%% file: arXiv.tex
\documentclass[letterpaper, 10 pt, conference]{ieeeconf}  % Comment this line out if you need a4paper

\IEEEoverridecommandlockouts                              % This command is only needed if 
\usepackage{censor}
\usepackage{graphics} % for pdf, bitmapped graphics files
\usepackage{amsmath} % assumes amsmath package installed
\usepackage{amssymb}  % assumes amsmath package installed

\usepackage{hyperref}
\usepackage{booktabs}   
\usepackage{array}
\usepackage{algorithm}
\usepackage{amsfonts}     
\usepackage{bbm}
\usepackage{caption}
\usepackage{subcaption}
\usepackage{comment}
\usepackage{tikz}

\usepackage{algpseudocode}
\newtheorem{assumption}{Assumption}
\newtheorem{theorem}{Theorem}
\newtheorem{remark}{Remark}

\usetikzlibrary{positioning, arrows, shapes.geometric}
\usetikzlibrary{arrows.meta, positioning, calc, shapes.geometric}

\title{\LARGE \bf
Integrating LTL Constraints into PPO for Safe Reinforcement Learning}

\author{Maifang Zhang$^{1,*}$, Hang Yu$^{2,*}$, Qian Zuo$^{1,*}$, Cheng Wang$^{3}$, Vaishak Belle$^{1}$, and Fengxiang He$^{1,\dag}$% <-this % stops a space
\thanks{$^{*}$Equal Contribution.}%
\thanks{$^{\dag}$Corresponding author. Email: \href{F.He@ed.ac.uk}{F.He@ed.ac.uk}.}
\thanks{$^{1}$M. Zhang, Q. Zuo, V. Belle, and F. He were with the School of Informatics, University of Edinburgh, Edinburgh EH8 9AB, Scotland.}
\thanks{$^{2}$H. Yu was with the School of Computer Science, Faculty of Engineering, University of Sydney, Darlington NSW 2008, Australia.}
\thanks{$^{3}$C. Wang was with the School of Engineering and Physical Sciences, Heriot-Watt University, Edinburgh EH14 4AS, Scotland.}
}
\begin{document}

\maketitle
\thispagestyle{empty}
\pagestyle{empty}

%%%%%%%%%%%%%%%%%%%%%%%%%%%%%%%%%%%%%%%%%%%%%%%%%%%%%%%%%%%%%%%%%%%%%%%%%%%%%%%%
\begin{abstract}
This paper proposes Proximal Policy Optimization with Linear Temporal Logic Constraints (PPO-LTL), a framework that integrates safety constraints written in LTL into PPO for safe reinforcement learning. LTL constraints offer rigorous representations of complex safety requirements, such as regulations that broadly exist in robotics, enabling systematic monitoring of safety requirements. Violations against LTL constraints are monitored by limit-deterministic Büchi automata, and then translated by a logic-to-cost mechanism into penalty signals. The signals are further employed for guiding the policy optimization via the Lagrangian scheme. Extensive experiments on the Zones and CARLA environments show that our PPO-LTL can consistently reduce safety violations, while maintaining competitive performance, against the state-of-the-art methods. The code is at \url{https://github.com/EVIEHub/PPO-LTL}. 
\end{abstract}

%%%%%%%%%%%%%%%%%%%%%%%%%%%%%%%%%%%%%%%%%%%%%%%%%%%%%%%%%%%%%%%%%%%%%%%%%%%%%%%%

\input{sections/MainText}

\input{sections/Appendices}

\section{Conclusion}

This paper introduces a PPO-LTL framework that augments PPO with safety specifications expressed in LTL through the Lagrangian scheme, providing a precise way to encode complex safety requirements, such as regulatory rules. %Constraint violations are tracked using limit-deterministic Büchi automata and then converted into penalty signals through a logic-to-cost mapping.
Theoretical guarantee on convergence is provided. Experiments in the Zones and CARLA environments are in full support of our method. %demonstrate that PPO-LTL consistently lowers safety violations while preserving competitive task performance, outperforming its counterparts. 

%%%%%%%%%%%%%%%%%%%%%%%%%%%%%%%%%%%%%%%%%%%%%%%%%%%%%%%%%%%%%%%%%%%%%%%%%%%%%%%%
%\newpage

%\section*{APPENDIX}

%Appendixes should appear before the acknowledgment.

%\section*{ACKNOWLEDGMENT}

%The preferred spelling of the word ÒacknowledgmentÓ in America is without an ÒeÓ after the ÒgÓ. Avoid the stilted expression, ÒOne of us (R. B. G.) thanks . . .Ó  Instead, try ÒR. B. G. thanksÓ. Put sponsor acknowledgments in the unnumbered footnote on the first page.

%%%%%%%%%%%%%%%%%%%%%%%%%%%%%%%%%%%%%%%%%%%%%%%%%%%%%%%%%%%%%%%%%%%%%%%%%%%%%%%%

\bibliographystyle{IEEEtran} \bibliography{Ref}

%\newpage

%\appendices
%\input{sections/Appendices}

\end{document}

%% file: sections/MainText.tex
\section{Introduction}

Reinforcement Learning (RL) has %been shown to 
achieved remarkable success across diverse domains, including robotics \cite{silver2016mastering}. %\fh{introduce PPO}  
In RL, Proximal Policy Optimization (PPO) is a widely adopted on-policy method, %due to its simplicity, stability, and strong empirical performance. 
which constrains policy updates with a clipped surrogate objective, striking a balance between exploration and stability \cite{schulman2017proximal}. %, which makes it a default choice for many continuous control benchmarks.
Deploying RL in safety-critical environments remains highly challenging, where violations against safety constraints can lead to catastrophic outcomes.
Safe RL has been comprehensively reviewed and addresses this challenge in the framework of constrained optimization,
%As a canonical model, constrained Markov decision process (CMDP)  \cite{altman2021constrained} formalizes this setting, 
where the agent seeks to maximize reward whilst bounding cumulative safety costs \cite{garcia2015comprehensive}. 
Within this family, constrained PPO via the Lagrangian scheme (PPO-Lagrangian)  has emerged as a major approach \cite{ray2019benchmarking}. %to adaptively balance performance and constraint satisfaction.
%\fhc{please differentiate citet and citep} 

Despite these advances, a critical limitation remains: the constraints need to be written in analytic inequalities of the agent's state and action. This is not compatible with a large family of abstract safety constraints, such as regulations that broadly exist in robotics. For example, the British Highway Code regulates driving behavior \cite{tennant2021code}, which is difficult to translate to the aforementioned inequalities.
%how to specify safety constraints in a principled and generalizable manner. 
%Traditional approaches often rely on manually designed penalties or heuristic reward shaping, which struggle to capture complex temporal safety requirements and frequently fail in rare but critical scenarios.
This calls for machine-computable and principled safety specifications within the RL training process.

\begin{comment}
    
\begin{figure}[t]
  \centering
  \includegraphics[width=0.4\textwidth]{sections/figures/fig1.pdf}
  \caption{%The illustration of PPO-LTL. 
  Top: Unconstrained policies lead to safety violations. %, including collisions, lane crossings, and traffic rule infractions. 
  Bottom: Safety requirements are formalized as LTL specifications, translated into differentiable cost signals in PPO-LTL.}  %through logic-to-cost monitors. These costs 
%  guiding PPO-LTL to produce a safe policy.} %that satisfies multiple safety constraints while maintaining task performance. The proposed approach bridges symbolic logic with safe reinforcement learning.} 
  %LTL constraints are translated into cost signals, combined with task rewards in a Lagrangian optimization scheme, and used to train a safety-compliant policy interacting with the environment.}
  \label{fig:illustration}
\end{figure}
\end{comment}

To address this issue, we propose Proximal Policy Optimization with Linear Temporal Logic Constraints (PPO-LTL), a novel method that represents abstract constraints in LTL \cite{pnueli1977temporal,baier2008principles}, and embeds these constraints into PPO agents. 
LTL provides a rigorous and machine-verifiable tool to encode temporal properties, such as ``always avoid unsafe states,'' ``eventually reach a goal,'', and regulatory rules, like ``stop at a red light until it turns green,'' as logic specifications. %\fhc{try to use highway code regulations}

%\fh{move to related works} 

We design a logic-to-cost mechanism that systematically translates violations of temporal logic constraints into cost signals that guide policy learning.
This mechanism can be instantiated in a wide range of environments, serving as a plug-and-play solution.
Each LTL specification is first compiled into an $\omega$-automaton, typically a limit-deterministic Büchi automaton (LDBA) \cite{vardi1996automata, sickert2016limit}. The automaton encodes the satisfaction conditions of the temporal property by defining a finite set of states and labeled transitions based on atomic propositions. During execution, it evolves synchronously with the agent-environment interaction, effectively acting as a runtime monitor that checks whether the agent’s trajectory satisfies the specification. %\fh{describe automata}

When a violation occurs, the monitor emits a cost signal to reflect the severity of the violation, determined by a set of pre-defined weights associated with different safety rules.
These violation costs are aggregated over time and integrated into the safe reinforcement learning framework.
The resulting signals are then combined with policy optimization via the Lagrangian scheme, allowing the agent to optimize task performance while ensuring that safety requirements are satisfied.
Unlike handcrafted penalties, this approach provides a principled, generalizable, and modular way to encode high-level safety requirements into the learning process.

%\fh{a paragraph on theory}
We provide a rigorous theoretical analysis of our approach. We formulate PPO-LTL as an inexact projected primal–dual method driven by biased stochastic gradient oracles. Specifically, we view two usual components in PPO, clipped surrogate objective and finite-epoch minibatch updates, as mechanisms that yield biased stochastic approximations to the true Lagrangian gradient. We then prove an ergodic stationarity guarantee for the projected primal-dual dynamics that underpin PPO-LTL. This result shows that despite the biased and noisy gradient estimates, our algorithm consistently converges to a neighborhood of a stationary point. Practically, it means PPO-LTL can stably reduce constraint violations without relying on exact gradient evaluations, highlighting its robustness in challenging settings such as autonomous driving.

We conduct comprehensive experiments to evaluate PPO-LTL across diverse benchmark environments, including ZonesEnv (continuous control with logical regions) \cite{ji2024omnisafe} and CARLA (autonomous driving simulator) \cite{dosovitskiy2017carla}. Our method is compared against PPO \cite{schulman2017proximal}, as the baseline, and a range of state-of-the-art safe RL methods, including TIRL-PPO, TIRL-SAC \cite{haarnoja2018soft}, PPO-Mask, PPO-Shielding \cite{alshiekh2018safe, jansen2020safe}, and PPO-Lagrangian. The empirical results demonstrate consistent reductions in safety violations while maintaining competitive task performance. Extensive ablation study and sensitivity analysis further verify our contributions. %The code is released anonymously.
%In particular, in the CARLA environment, PPO-LTL reduces the collision rate by 45.4\% compared to vanilla PPO and by 65.0\% compared to PPO-Shielding, highlighting the substantial safety gains achieved by our logic-constrained approach. 
%The code is at \url{https://github.com/eviehub/PPO-LTL}.
%These results demonstrate the effectiveness of combining symbolic logic with constrained policy optimization for Safe RL.
%The code is available at \url{https://anonymous.4open.science/r/PPO-LTL-C228/}.

%\textbf{Safe Reinforcement Learning.} 

%\textbf{Safety Mechanisms in RL.}

%The literature of safe reinforcement learning has been comprehensively surveyed by \cite{garcia2015comprehensive}, and spans several families: 
%\fh{please make sure all relevant papers are cited}
%\textbf{Regularization for safe RL.} Regularization methods encode safety by penalties or constrained surrogates. A representative example is reward-constrained policy optimization \cite{tessler2018reward}, which tunes penalty coefficients through dual variable updates. However, this line of approaches is sensitive to coefficient choices and can fail under rare catastrophic events. %\fhc{references}

\section{Related Work}

\textbf{Shielding.} 
Shielding approaches enforce safety by preempting unsafe actions online using verified policies or model-checking over abstract models \cite{alshiekh2018safe,jansen2020safe}. This yields binary-safe behavior with strong formal guarantees but can restrict exploration and lead to non-stationary data distributions. In contrast, soft-integration methods map violations into cost signals through a logic-to-cost mechanism, providing dense feedback that is compatible with gradient-based optimization and constrained MDP solvers \cite{ray2019benchmarking, altman2021constrained}. Using LTL monitors to emit per-rule costs further enables modular handling of multiple constraints, compositional reasoning across specifications, and straightforward scalability to large rule sets \cite{baier2008principles, camacho2019ltl}. Compared with runtime action filtering strategies, PPO-LTL incorporates temporal logic constraints directly within the policy optimization loop, providing dense feedback signals that support learning under complex temporal requirements. %\fh{discuss our advantages}

\textbf{Logic tools in RL.} 
%\fh{shorten it} 
Beyond LTL, other formal logics have also been explored in reinforcement learning. 
Signal Temporal Logic (STL)~\cite{maler2004monitoring} extends temporal reasoning to real-valued signals, while deontic logics~\cite{chan2024formalise} express normative concepts like obligations and permissions. 
These approaches offer richer expressiveness but often come with higher computational costs and task-specific design complexity, limiting their practicality. 
Probabilistic Logic Programming (PLP)~\cite{de2015probabilistic} has also been introduced to model uncertainty in logical reasoning, allowing agents to encode probabilistic constraints. 
In contrast, LTL achieves a balance: it is lightweight, easy to compile into automata, and sufficient to capture temporal safety properties~\cite{chen2021interpretable}. Recent work~\cite{qi2025risk} applies LTL in a model-based driving framework via product MDPs and linear programming for policy synthesis, while our method focuses on learning-based constrained optimization rather than model-based synthesis.
%This efficiency-expressiveness balance makes LTL the natural choice for PPO-LTL, enables efficient constraining during policy optimization with mild sacrifice of the ability to specify complex temporal safety requirements.
%These characteristics make LTL a natural foundation for our safe RL framework.

%\textbf{Runtime Monitoring and Integration.}

%\textbf{Shielding.}  \textit{Hard integration} (shielding) filters or corrects actions online to avoid specification violations, as proposed by \cite{alshiekh2018safe} and extended by \cite{jansen2020safe}, yielding binary-safe behavior but potentially inducing conservatism and non-stationary data.  \textit{Soft integration} instead maps violations to costs via a logic-to-cost mechanism, producing dense feedback that is amenable to gradient-based optimization and compatible with CMDP solvers, as discussed by \cite{ray2019benchmarking} and \cite{altman2021constrained}.  Using LTL monitors to emit per-rule costs further enables modular multi-constraint handling, compositionality across specifications, and straightforward scaling to many rules, as shown by \cite{baier2008principles} and \cite{camacho2019ltl}.
%(Camacho et al., 2019; Baier \& Katoen, 2008).

\begin{figure*}[t!]
  \centering
  \includegraphics[width=0.8\textwidth]{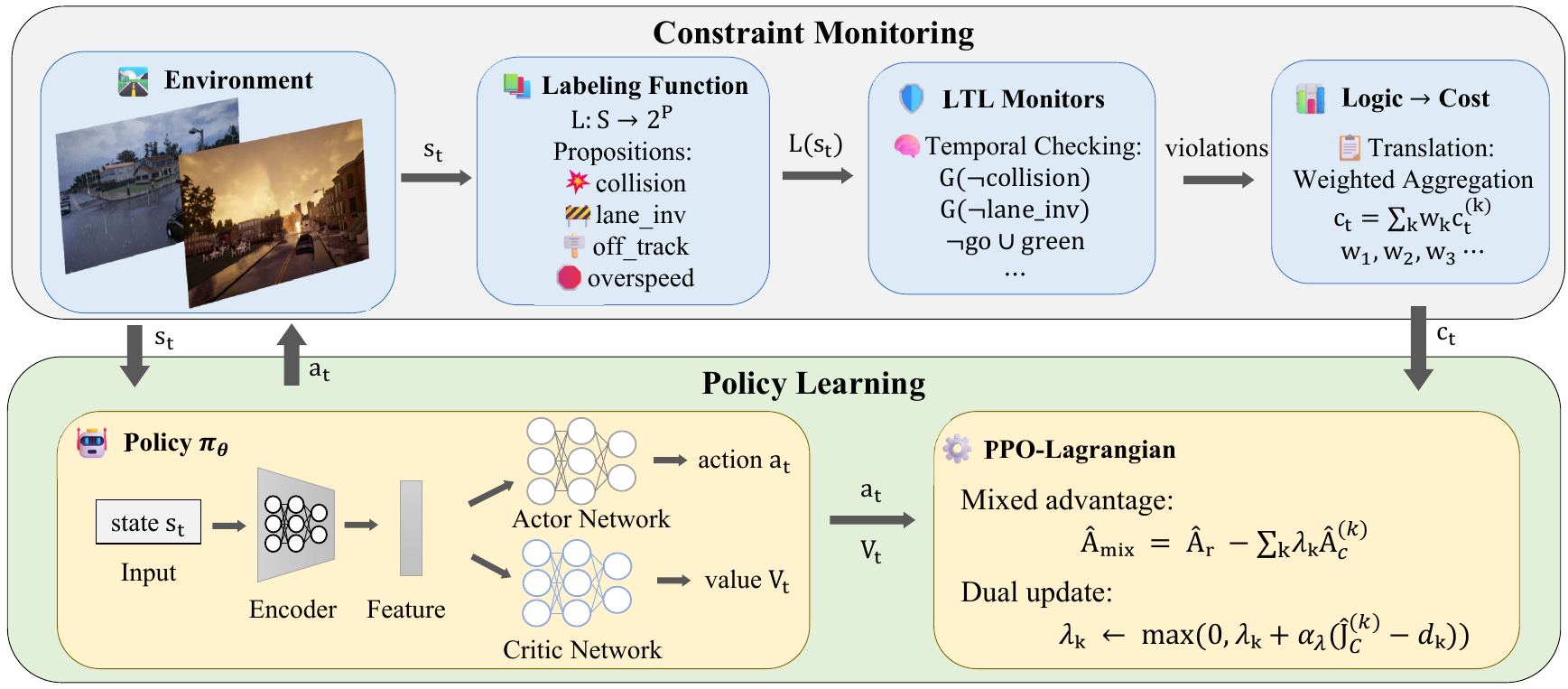}
  \caption{PPO-LTL: environment states are labeled with atomic propositions, monitored by LTL checkers to generate constraint costs, which are integrated with task rewards for policy optimization.} %\fh{can be better}
  \label{fig:framework}
\end{figure*}

\section{Preliminaries}

\textbf{Markov Decision Process (MDP).} \label{sec:cmdp}
We consider a discounted constrained MDP 
$\mathcal{M}=(\mathcal{S},\mathcal{A},P,r,c,\mu,\gamma)$, 
where $\mathcal{S}$ and $\mathcal{A}$ denote the state and action spaces, respectively, 
$P(s'| s,a)$ represents the transition probability of moving from state $s$ to state $s'$ given action $a$, 
$r:(s,a)\rightarrow[0,R_\text{max}]$ is the reward function,  
$c:(s,a)\rightarrow [0,C_\text{max}]$ is the aggregated cost  
\(
c(s,a)=\sum_{i=1}^K w_i c^{(i)}(s,a)
\)
where $w_i$ are fixed weights for $i\in[K]$, 
$\mu$ is the initial-state distribution, and 
$\gamma \in (0,1)$ is the discount factor controlling how much the agent values future rewards.
A policy $\pi: \mathcal{S}\rightarrow \mathcal{P}(\mathcal{A})$ 
is a mapping from the state space to the space of probability distribution over the action space and 
$\pi(a\mid s)$ denotes the probability of selecting action $a$ at state $s$. 
Let $\theta\in \Theta \subset \mathbb{R}^p$ denote the parameters of a stochastic policy $\pi_\theta(\cdot|s)$.
For a given policy $\pi_\theta$, define the discounted occupancy measure over $\mathcal{S}\times\mathcal{A}$ as
\(
d_{\theta}^\gamma
=
(1-\gamma)
\sum_{t=0}^\infty 
\gamma^t 
\mathrm{P}_{\pi_\theta}(s_t=s,a_t=a|\mu)
\).
We define reward value function and cost value function as,
\[
J_R(\theta)
=
\mathbb{E}_{(s,a)\sim d_{\theta}^\gamma}
\left[r(s,a)\right], \text{ }
%\]
%and the cost value function as
%\[
J_C(\theta)
=
\mathbb{E}_{(s,a)\sim d_{\theta}^\gamma}
\left[c(s,a)\right].\]
%\]
The constrained optimization objective is as follows,
\[
\max_{\theta\in\Theta} J_R(\theta)
\quad 
\text{s.t.} 
\quad 
J_C(\theta)\le d,\]
%\]
where $d=\sum_{i=1}^K w_i d_i$ is the aggregated safety budget.
The corresponding Lagrangian function is defined as
\(
\mathcal{L}(\theta,\lambda)
=
J_R(\theta)
-
\lambda(J_C(\theta)- d)\),
%\]
where $\lambda\in[0,\Lambda]$ is the Lagrangian multiplier with $\Lambda<\infty$.

\textbf{Proximal Policy Optimization (PPO).}
Policy gradient methods directly optimize the policy but could suffer from instability if updates are too large \cite{sutton2018reinforcement}. %\fhc{reference} 
%As shown by , 
PPO %\cite{schulman2017proximal} 
addresses this issue by clipping updates, effectively putting a ``safety belt'' on learning.
The clipped surrogate objective $L^{\text{PPO}}(\pi)$ is defined as follows,
\[
\mathbb{E}_t \Big[ 
\min \big( \rho_t(\pi) A_t, \; \text{clip}(\rho_t(\pi), 1-\epsilon, 1+\epsilon) A_t \big) 
\Big],
\]
where $\rho_t(\pi) = \tfrac{\pi(a_t|s_t)}{\pi_{\text{old}}(a_t|s_t)}$ is the importance ratio, 
$A_t$ is the advantage estimate (measuring how good an action is compared to the average), 
and $\epsilon$ is a clipping parameter that prevents drastic updates. In practice, PPO optimizes this objective using multiple epochs of mini-batch stochastic gradient updates over collected rollouts. %\fh{minibatch}

% \fh{introduce basics, including the grammar and automata}

\textbf{Linear Temporal Logic (LTL).} 
LTL is a formal language used to specify temporal properties over infinite sequences of system states. %\cite{pnueli1977temporal,baier2008principles}. 
A formula in LTL, called a specification, defines a desired temporal behavior of the system. An LTL formula is constructed from a finite set of atomic propositions $\mathcal{P}$, which represent fundamental facts about the environment (e.g., ``collision occurred'' or ``goal reached''). 
Each proposition is a Boolean statement about the system state that is either true or false, such as whether the vehicle is currently in a safe zone or whether a traffic light is green. 
These propositions are combined using Boolean connectives ($\neg$, $\land$, $\lor$) and temporal operators such as $\mathbf{G}$ (always), $\mathbf{F}$ (eventually), $\mathbf{X}$ (next), and $\mathbf{U}$ (until) to express more complex requirements over time. 
The semantics of an LTL formula specify the precise conditions under which a specification is considered satisfied. They are defined over a trajectory $\sigma = s_0 s_1 s_2 \dots$, where $\sigma, t \models \varphi$ indicates that the specification $\varphi$ holds at time $t$ — in other words, the system’s behavior up to time $t$ conforms to the property described by $\varphi$. 
For example, $\mathbf{G}\neg\texttt{collision}$ requires that collisions never occur, while $\mathbf{F}\,\texttt{goal}$ means the goal state must eventually be reached. 

\section{PPO-LTL}

%This section describes our algorithm. %, Proximal Policy Optimization with Linear Temporal Logic Constraints (PPO-LTL).

\subsection{Constraints as LTL Specifications}

In real-world applications, many safety requirements depend simultaneously on the environment’s state (e.g., position, distance to obstacles) and the temporal structure of events (e.g., their order and timing) \cite{garcia2015comprehensive,ray2019benchmarking,camacho2019ltl,li2023temporal}.
For example, requirements like “reach the goal eventually after visiting a checkpoint” cannot be simply encoded as scalar penalties. They require temporal reasoning over sequences of states \cite{vaezipoor2021ltl2action,chen2021interpretable}.
To formally represent such requirements, we define them as LTL specifications. 
Each specification is a logical formula describing desired temporal behavior over system trajectories and serves as a formal requirement that the agent’s policy should satisfy. 
A pilot work has been in the literature that represents traffic laws as LTL specifications. %\fhc{Formalise regulations for autonomous vehicles with right-open temporal deontic defeasible logic}.
LTL combines temporal operators with Boolean connectives to specify complex behaviors, e.g.,
always avoiding collisions while eventually reaching a destination
($\mathbf{G}\neg\texttt{collision} \land \mathbf{F}\,\texttt{destination}$),
or requiring that entering an intersection is eventually followed by a green light
($\mathbf{G}(\texttt{intersection} \rightarrow \mathbf{F}\,\texttt{green})$).
This expressiveness enables compact modeling of multi-stage driving rules in autonomous driving scenarios.

%At runtime, specifications are usually enforced by monitors — simple components that observe how the environment evolves over time and check whether the current trajectory meets the required properties. These monitors operate step by step and generate a violation signal whenever the agent’s behavior deviates from the specification.

\subsection{Logic-to-Cost Mechanism}% via Büchi Automata and Monitors}

%For practical use in reinforcement learning, a simplified form called a limit-deterministic Büchi automaton (LDBA) is often employed, which restricts nondeterminism to an initial phase and transitions deterministically thereafter, improving runtime efficiency.

%: DeepLTL's 

%\fh{consider move zones and carla into appendix or experiments; rather, two others are missed.}

%\fh{give more illustrations}

\textbf{Büchi Automata and LDBA.}
Each LTL specification can be translated into a Büchi automaton (BA), a state-transition structure that monitors whether the agent’s event sequence satisfies the temporal rule.
The automaton reads the interaction trajectory, and satisfaction is achieved when designated accepting states are visited infinitely often during execution.
For reinforcement learning, we adopt a simplified variant called the limit-deterministic Büchi automaton (LDBA), which provides more predictable checking and improved computational efficiency.
The LDBA enables high-level temporal logic to be evaluated step by step during training, forming the basis for converting symbolic rules into numerical signals that guide policy learning.

\textbf{Monitors.} During training, temporal specifications are checked by runtime monitors that evolve synchronously with the environment.
Each monitor observes the trajectory and determines whether a specification $\phi_i$ is satisfied.
When a violation-related transition is detected, the monitor emits a nonnegative cost signal $c_t^{(i)}$, whose magnitude is determined by a rule-specific weight reflecting its relative importance.
Multiple monitors may generate costs simultaneously, and all violation costs are aggregated to guide policy optimization in the CMDP framework:
$c_t = \sum_{i=1}^K c^{(i)}_t$.
Safety-critical rules contribute larger costs, while goal-oriented requirements maintain elevated costs until the condition is satisfied.

\textbf{Reach-Avoid Decomposition.}
To further simplify policy optimization, each compiled automaton can be decomposed into a sequence of reach-avoid subtasks \cite{vaezipoor2021ltl2action,jackermeier2025deepltl}. 
Each subtask consists of (1) a reach condition: the state or event that must eventually occur, and (2) an avoid condition: the event that must never happen. 
For example, $\mathbf{F}(\texttt{goal}) \wedge \mathbf{G}\neg\texttt{collision}$ is decomposed into two subtasks: “always avoid collisions” and “eventually reach the goal.” 
This transformation reduces policy search complexity while preserving the original temporal semantics.

%To further simplify optimization, complex LTL formulas are decomposed into a sequence of reach-avoid subtasks \cite{vaezipoor2021ltl2action,jackermeier2025deepltl}. Each subtask specifies a goal that must eventually be achieved (reach) and a set of conditions that must always be avoided (avoid). For example, $\mathbf{F}(\texttt{goal}) \wedge \mathbf{G}\neg\texttt{collision}$ is reduced to two components: (1) always avoid collisions, and (2) eventually reach the goal region. This decomposition significantly reduces the policy search space while preserving the semantics of the original temporal logic.

%\subsection{Logic-to-Cost Mechanism}

%We now have encoded safety requirements as LTL formulas defined over atomic propositions that describe high-level events in the environment (e.g., ``collision occurred'', ``entered safe zone''). Each environment state is labeled with the corresponding propositions. 

%Any violations against the LTL specifications are transformed into per-step nonnegative cost signals that quantify the degree of violation. 

%During training, the environment returns $(s_{t+1}, r_t, c_t, \texttt{info})$, where $r_t$ represents task progress and $c_t$. %can be used as the constraint channel consumed by our PPO-LTL, which will be described in the following subsection. 
%This unified mechanism converts symbolic temporal requirements into structured numerical signals that guide policy optimization while ensuring compliance with safety constraints.

\textbf{Logic-to-Cost Mechanism.} The final environment feedback is therefore given by:
%\[
$(s_{t+1}, r_t, c_t, \texttt{info})$,
%\]
where $r_t$ encodes task performance, $c_t$ represents aggregated constraint costs, and \texttt{info} provides diagnostic information for each rule. This runtime logic-to-cost conversion is domain-agnostic and can be applied across diverse environments, %such as ZonesEnv and CARLA, 
guiding reinforcement learning toward policies that satisfy temporal logic constraints while optimizing performance.

\subsection{The Langragian Scheme in PPO-LTL}

%We adopt PPO given per-step costs derived from the Logic-to-Cost model. \fh{The agent seeks to maximize the expected discounted rewards while ensuring that the expected discounted costs remain below specified limits. This is formalized as\[\max_{\pi} J_R(\pi) \quad \text{s.t.} \quad J_C^{(k)}(\pi) \le d_k, \;\; \forall k \in \{1,\dots,K\},\]where $d_k$ is the safety budget for the $k$-th constraint. We apply Lagrangian relaxation by defining\[L(\pi, \lambda) = J_R(\pi) - \sum_{k=1}^K \lambda_k \big(J_C^{(k)}(\pi) - d_k\big),\]where $\lambda_k \geq 0$ are dual variables. } \fh{define it using our terminologies}
Given the per-step violation costs produced by the logic-to-cost mechanism, PPO-LTL incorporates them directly into the policy optimization process. 
%\fh{Instead of explicitly solving the constrained optimization problem described in Section~\ref{sec:cmdp},} we adopt a primal-dual approach 
We solve the constrained optimization problem %described in Section~\ref{sec:cmdp} 
using a primal-dual approach, where constraint information influences policy updates through a mixed advantage signal
\(
\hat A_{\text{mix}} = \hat A_r - \sum_{k=1}^K \lambda_k \hat A_c^{(k)}
\),
where $\hat A_r$ and $\hat A_c^{(k)}$ are generalized advantage estimates for reward and cost, respectively. 
After each PPO update, the multipliers are updated via projected gradient ascent:
\[
\lambda_k \leftarrow \max\Big(0,\ \lambda_k + \alpha_\lambda \big(\hat J_C^{(k)} - d_k\big)\Big).
\]
%This scheme allows the policy to adaptively balance safety constraints and reward optimization.
When costs exceed their pre-defined limits, $\lambda_k$ increases, strengthening the penalty applied to violations.
Conversely, when the costs remain within acceptable bounds, $\lambda_k$ decreases or stays constant, enabling the policy to continue improving task performance.
%Thus, PPO-LTL integrates symbolic logic constraints with gradient-based policy optimization in a principled and adaptive manner.

%\subsection{Theoretical Motivations}

\section{Theoretical Guarantee}

In this section, we analyze the convergence properties of PPO-LTL. %To account for the temporal nature of LTL constraints, w
We formulate the learning process within a Product MDP framework: %. Specifically, 
the state space is treated as an augmented space $\mathcal{S} = \mathcal{S}_{env} \times \mathcal{Q}$, where $\mathcal{S}_{env}$ is the environment state, and $\mathcal{Q}$ is the LDBA state. %By augmenting the state space, t
The temporally dependent LTL cost becomes strictly Markovian, allowing %us to 
defining the cost function as $c(s,a) = \sum_{k=1}^K w_k c^{(k)}(s,a)$. %Under this formulation
Correspondingly, PPO-LTL generates an iterate sequence $\{(\theta_t,\lambda_t)\}_{t\ge 0}$ as follows:
    %PPO-LTL generates an iterate sequence $\{(\theta_t,\lambda_t)\}_{t\ge 0}$ as follows:
    \begin{equation}\label{eq1}
        \theta_{t+1}=\Pi_{\Theta}(\theta_t+\alpha\hat{g}_t), \quad \lambda_{t+1}=\Pi_{[0,\Lambda]}(\lambda_t+\beta\hat{u}_t),
    \end{equation}
    where $\Pi$ denotes Euclidean projection, $\alpha$, $\beta>0$ are the learning rates, $\hat{g}_t$ is a stochastic ascent direction for the primal objective, and $\hat{u}_t$ is a stochastic estimate of the constraint residual $\hat{J}_C(\theta_t)-d$. %Moreover, the 
%    The policy-gradient signal is computed using a mixed advantage
%    \begin{equation*}
%$        \hat{A}_t^\text{mix}(s,a)=\hat{A}_t^r(s,a)-\lambda_t\hat{A}_t^c(s,a)$,
%    \end{equation*}
%    where $\hat{A}_t^r$ and $\hat{A}_t^c$ are reward and cost advantage estimates.
    We further abstract PPO-LTL in an inexact projected primal-dual framework with biased stochastic gradient oracles: %. In particular, 
    the clipped surrogate optimization and finite-epoch minibatch updates in PPO are modeled as producing biased stochastic estimates of the true Lagrangian gradient.
    For  given learning rates $\alpha, \beta>0$, define the primal and dual gradient mappings as follows:
    \begin{align*}
        \mathcal{G}(\theta,\lambda):=&\frac{1}{\alpha}\left(\Pi_{\Theta}\left(\theta+\alpha\nabla_\theta\mathcal{L}(\theta,\lambda)\right)-\theta\right),\\
        \mathcal{H}(\theta,\lambda):=&\frac{1}{\beta}\left(\Pi_{[0,\Lambda]}\left(\lambda+\beta(J_C(\theta)-d)\right)-\lambda\right),
    \end{align*}
    %These mappings 
    quantifying first-order stationarity of this problem under projection. 
    
    The theory relies on two mild assumptions on the primal domain and stochastic gradient, given in Sec \ref{appendix_theory}, following usual practice \cite{baxter2001infinite,schulman2017proximal,achiam2017constrained}. We prove the following theorem. 
    \begin{theorem}\label{thm1}
        Conditioned on Assumptions \ref{assume1} and \ref{assume2} and the learning rates $0<\alpha\le 1/(4L_{\mathcal{L}})$ and $\beta>0$, let $\{\theta_t,\lambda_t\}_{t\ge0}$ be defined by \eqref{eq1}. Then, for all $T\ge1$,
        \begin{align*}
            &\frac{1}{T}\sum_{t=0}^{T-1}\mathbb{E}[\|\mathcal{G}(\theta_t,\lambda_t)\|]\\\le&\sqrt{\frac{2(\Delta_{\mathcal{L}}+2\Lambda U_{\max})}{\alpha T}}+O\left(\sqrt{\sigma_{\theta}^2+a^2+\alpha}\right)\\
            &+O\left(\left(G_\text{max}^2+\sigma_{\theta}^2+a^2\right)^{1/4}\right),\\
            &\frac{1}{T}\sum_{t=0}^{T-1}\mathbb{E}[\|\mathcal{H}(\theta_t,\lambda_t)\|]\\\le& \sqrt{\frac{\Delta_\mathcal{L}}{\beta T}}+O\left(\sqrt{{U_{\max}^2+\sigma_{\lambda}^2+b^2}+\alpha^2+\frac{G_{\max}^2}{\beta}}\right)\\
            &+O\left(\frac{\alpha}{\sqrt{\beta}}\sqrt{G_{\max}^2+\sigma_{\theta}^2+a^2}\right),
        \end{align*}
        where $\Delta_\mathcal{L}=\sup_{\Theta\times[0,\Lambda]}\mathcal{L}(\theta,\lambda)-\inf_{\Theta\times[0,\Lambda]}\mathcal{L}(\theta,\lambda)<\infty$, $U_{\max} =\sup_{\theta\in\Theta}\left|J_C(\theta)-d\right|$, $G_{\max}=\sup_{(\theta,\lambda)}\|\nabla_\theta\mathcal{L}(\theta,\lambda)\|$, and $O(\cdot)$ hides problem-dependent constants independent of $T$.
    \end{theorem}

%\textbf{Remark.} 

A detailed proof is given in Section~\ref{appendix_theory}.
Theorem \ref{thm1} establishes an ergodic stationarity guarantee for the projected primal-dual dynamics underlying PPO-LTL. It demonstrates that despite the biased and noisy gradient estimates inherent to PPO (e.g., due to clipping and minibatch sampling, captured by the variance and bias terms $\sigma$ and $a, b$), the algorithm reliably converges to a neighborhood of the stationary point. In practice, this implies that PPO-LTL can stably minimize constraint violations without requiring exact gradient computation, confirming its robustness in complex environments like autonomous driving.

\section{Experiments}

%\subsection{Assumptions} 
%We assume that the environment can be modeled as a CMDP and that the agent has full access to the relevant state variables needed to define atomic propositions. 
%We further assume that safety and task objectives can be expressed as LTL specifications over these propositions, and that they can be compiled into equivalent LDBAs whose transitions evolve synchronously with the environment dynamics. 
%Finally, we assume that the resulting cost signals are suitable for gradient-based optimization and that Lagrangian dual updates converge under standard conditions.

%\fh{a paragraph here}

We conduct extensive experiments on ZonesEnv and CARLA, which fully support our algorithm. 
\textbf{Baselines and Comparison Methods.}
PPO is used as an unconstrained baseline. 
TIRL-PPO and TIRL-SAC are included as standard Safe RL baselines to evaluate the performance of alternative constrained optimization techniques. 
PPO-Mask is a heuristic safety filter that preemptively overrides imminent unsafe actions with predefined safe fallbacks (e.g., hard braking). We include this to illustrate the limitations (e.g., deadlocks and over-conservatism) of purely reactive, rule-based interventions. 
PPO-Shielding is our main competitor in the experiments.
PPO-Lagrangian is included as a standard constrained RL method. 
For fair comparison, all methods use a CNN backbone, consisting of 6 convolutional layers with ReLU activations, followed by policy and value networks with 2-layer MLPs containing [500, 300] units each. 

%\textbf{ZonesEnv.} 
%Full environment configuration and network details are provided in Appendix %~\ref{app:implementation}.
%We further validate PPO-LTL in the CARLA 
%and a
%Detailed settings and training pipelines are described in Appendix %~\ref{app:implementation}.
%Network architectures and curriculum learning schedules follow standard PPO settings. 
%LTL monitors map events (e.g., collisions, lane invasions) to deterministic cost signals, which are integrated with rewards during PPO-Lagrangian updates. 
\begin{comment}
\textbf{Environments.} We employ (1) \textbf{ZonesEnv:} a logical grid-world environment (ZonesEnv) built on Safety Gymnasium \cite{ji2024omnisafe}, where the agent (a point robot) navigates a 2D plane divided into colored zones representing atomic propositions. LTL specifications encode objectives such as ``avoid blue until reaching green and eventually visit yellow.'' Wall collisions terminate episodes with penalties, and violation costs are generated based on whether the trajectory satisfies the specification. 
(2) \textbf{CARLA:} an autonomous driving simulator \cite{dosovitskiy2017carla}, using Town02 with realistic intersections, traffic lights, and dynamic vehicles. Observations include semantic segmentation and ego-state features, and the one-hot encoded LDBA state to ensure the Markov property.
Actions control steering, throttle, and brake. 
Episodes terminate upon collision, large deviation, or route completion. CARLA experiments used CARLA 0.9.13 with API and off-screen rendering. 
\end{comment}

\begin{table}[t]
\centering
\caption{Results in ZonesEnv.} 
\label{tab:zones}
\resizebox{0.47\textwidth}{!}{
\begin{tabular}{l|ccc}
\toprule
\textbf{Method} & \textbf{Reward} & \textbf{Hit Wall Rate} & \textbf{$\lambda$} \\
\midrule
PPO & 17.95 $\pm$ 0.84 & 3.7 $\pm$ 2.1\% & - \\
PPO-Mask & 10.14 $\pm$ 2.82 & 6.0 $\pm$ 2.0\% & - \\
PPO-Shielding & 15.92 $\pm$ 0.55 & 12.0 $\pm$ 2.0\% & - \\
PPO-Lagrangian & 23.23 $\pm$ 1.51 & 6.0 $\pm$ 2.0\% & 0.00 \\
\midrule
PPO-LTL-A & 17.86 $\pm$ 0.73 & 4.3 $\pm$ 3.5\% & 0.0048 \\
PPO-LTL-B & {18.61 $\pm$ 0.67} & 4.7 $\pm$ 3.1\% & 0.0018 \\
\bottomrule
\end{tabular}
}
\end{table}

% ========== Main Training Results Table ==========
\begin{table*}[t!]
\centering
\caption{Results in CARLA. Arrows ($\uparrow, \downarrow$) indicate the direction where values are strictly better.} 
\label{tab:main_training}
\resizebox{0.9\textwidth}{!}{
\begin{tabular}{l|cc|ccccccc}
\toprule
\textbf{Methods} & \textbf{Collision Rate $\downarrow$} & \textbf{Routes $\uparrow$} & \textbf{Distance} & \textbf{Speed} & \textbf{Length} & \textbf{Col. Num} & \textbf{Cent. Dev} & \textbf{Cost} & \textbf{$\lambda$} \\
\midrule
PPO & 0.262 $\pm$ 0.115 & 0.013 $\pm$ 0.008 & 6.58 & 0.45 & 4030 & 8.0 & 0.77 & - & - \\
TIRL-PPO & 0.173 $\pm$ 0.129 & 0.083 $\pm$ 0.111 & 6.37 & 0.07 & 9460 & 7.3 & 0.87 & - & - \\
TIRL-SAC & 0.336 $\pm$ 0.063 & 0.027 $\pm$ 0.011 & 6.92 & 0.92 & 926 & 57.3 & 0.12 & - & - \\
PPO-Mask & 0.408 $\pm$ 0.058 & 0.010 $\pm$ 0.012 & 2.38 & 1.06 & 1338 & 38.0 & 0.37 & - & - \\
PPO-Shielding & 0.267 $\pm$ 0.056 & 0.072 $\pm$ 0.036 & 18.87 & 10.63 & 93 & 164.3 & 0.52 & - & - \\
PPO-Lagrangian & 0.233 $\pm$ 0.089 & 0.077 $\pm$ 0.051 & 7.78 & 0.96 & 6616 & 103.0 & 0.60 & 0.00 & 0.01 \\
\midrule
PPO-LTL-A & \textbf{0.143 $\pm$ 0.110} & 0.077 $\pm$ 0.107 & 5.96 & 3.57 & 703 & 154.3 & 0.52 & 0.25 & 0.03 \\
PPO-LTL-B & 0.170 $\pm$ 0.148 & \textbf{0.236 $\pm$ 0.037} & 12.79 & 1.48 & 4859 & 10.0 & 0.73 & 0.08 & 0.00 \\
\bottomrule
\end{tabular}
}
\end{table*}

% ========== Comprehensive Ablation, Sensitivity, and Extensive Configurations Table ==========
\begin{table*}[t!]
\centering
\caption{Ablation studies and sensitivity analyses in CARLA.} %Arrows ($\uparrow, \downarrow$) indicate ndicate the direction where values are strictly better.}
%(cost limits and dual update rates), diverse constraint formulations (including standard Lagrangian baselines). 
% core performance metrics. Other metrics characterize driving behavior and are reported as means.}
\label{tab:comprehensive_eval}
\resizebox{0.9\textwidth}{!}{
\begin{tabular}{l|cc|ccccccc}
\toprule
\textbf{Configuration} & \textbf{Collision Rate $\downarrow$} & \textbf{Routes $\uparrow$} & \textbf{Distance} & \textbf{Speed} & \textbf{Length} & \textbf{Col. Num} & \textbf{Cent. Dev} & \textbf{Cost} & \textbf{$\lambda$} \\
\midrule
\multicolumn{10}{c}{\textit{Ablation of LTL Components}} \\
\midrule
No collision & \textbf{0.159 $\pm$ 0.027} & 0.009 $\pm$ 0.003 & 2.71 & 0.68 & 4123 & 11.0 & 0.64 & 0.07 & 0.00 \\
No off-track & 0.275 $\pm$ 0.129 & 0.032 $\pm$ 0.044 & 3.29 & 0.06 & 10501 & 5.7 & 0.62 & 0.07 & 0.00 \\
No lane invasion & 0.207 $\pm$ 0.046 & \textbf{0.024 $\pm$ 0.009} & 6.17 & 1.98 & 14323 & 21.7 & 0.57 & 0.06 & 0.01 \\
\midrule
\multicolumn{10}{c}{\textit{Sensitivity of Cost Limit (cl) and Dual Learning Rate ($\alpha_\lambda$)}} \\
\midrule
cl = 0.001 & 0.229 $\pm$ 0.153 & 0.033 $\pm$ 0.035 & 5.19 & 1.10 & 4749 & 21.0 & 0.78 & 0.06 & 0.01 \\
cl = 0.05 & 0.273 $\pm$ 0.199 & 0.020 $\pm$ 0.023 & 3.04 & 4.65 & 907 & 228.7 & 0.56 & 0.22 & 0.71 \\
cl = 0.5 & 0.395 $\pm$ 0.130 & 0.045 $\pm$ 0.059 & 4.86 & 0.09 & 6614 & 11.7 & 0.91 & 0.01 & 0.00 \\
%\midrule
%\multicolumn{10}{c}{\textit{Sensitivity of Dual Learning Rate ($\alpha_\lambda$)}} \\
%\midrule
$\alpha_\lambda$ = 0.00001 & \textbf{0.160 $\pm$ 0.059} & 0.062 $\pm$ 0.063 & 8.67 & 0.89 & 2957 & 11.7 & 0.95 & 0.01 & 0.00 \\
$\alpha_\lambda$ = 0.0001 & 0.243 $\pm$ 0.081 & 0.053 $\pm$ 0.033 & 8.98 & 0.54 & 3029 & 9.7 & 0.63 & 0.07 & 0.00 \\
$\alpha_\lambda$ = 0.01 & 0.233 $\pm$ 0.097 & \textbf{0.099 $\pm$ 0.065} & 8.28 & 2.05 & 3776 & 8.3 & 0.86 & 0.06 & 0.03 \\
\midrule
\multicolumn{10}{c}{\textit{Constraint Strictness \& Mixed Formulations}} \\
\midrule
Relaxed high & 0.258 $\pm$ 0.067 & 0.042 $\pm$ 0.040 & 4.51 & 0.09 & 3380 & 13.3 & 0.62 & 0.01 & 0.00 \\
Relaxed moderate & 0.205 $\pm$ 0.100 & 0.056 $\pm$ 0.053 & 3.79 & 1.50 & 2787 & 53.7 & 0.53 & 0.01 & 0.00 \\
Collision focused & 0.233 $\pm$ 0.206 & 0.010 $\pm$ 0.008 & 0.07 & 0.22 & 3020 & 15.7 & 0.39 & 0.00 & 0.00 \\
Ultra loose & 0.393 $\pm$ 0.110 & 0.141 $\pm$ 0.098 & 8.81 & 1.04 & 3996 & 18.0 & 0.77 & 0.00 & 0.00 \\
Mixed light & 0.282 $\pm$ 0.076 & 0.042 $\pm$ 0.056 & 3.37 & 0.64 & 4218 & 18.0 & 0.90 & 0.01 & 0.00 \\
Mixed light loose & 0.242 $\pm$ 0.018 & 0.041 $\pm$ 0.033 & 2.86 & 0.26 & 5728 & 7.0 & 0.48 & 0.01 & 0.00 \\
Mixed medium & 0.229 $\pm$ 0.131 & 0.044 $\pm$ 0.029 & 5.62 & 3.69 & 3248 & 50.7 & 0.81 & 0.06 & 0.00 \\
\midrule
\multicolumn{10}{c}{\textit{Collision-Only Variants Without Temporal LTL}} \\
\midrule
Col-only (cl=0.5) & 0.243 $\pm$ 0.079 & 0.033 $\pm$ 0.043 & 4.93 & 0.52 & 4155 & 11.3 & 1.04 & 0.00 & 0.00 \\
Col-only (cl=0.1) & 0.248 $\pm$ 0.228 & 0.028 $\pm$ 0.030 & 6.10 & 0.40 & 3573 & 14.3 & 0.90 & 0.00 & 0.00 \\
Col-only mid & 0.346 $\pm$ 0.175 & 0.158 $\pm$ 0.200 & 9.07 & 2.07 & 1584 & 52.0 & 0.86 & 0.01 & 0.00 \\
Col-only loose & 0.193 $\pm$ 0.081 & \textbf{0.269 $\pm$ 0.297} & 12.66 & 0.93 & 9531 & 9.3 & 0.56 & 0.00 & 0.00 \\
\bottomrule
\end{tabular}
}
\end{table*}

\textbf{Environments and Implementation.} We evaluate on two environments using a 256-dimensional CustomMultiInputExtractor \cite{raffin2021stable} and 3 random seeds.
(1) \textbf{ZonesEnv:} A Safety Gymnasium \cite{ji2024omnisafe} grid-world where a point robot navigates colored zones representing atomic propositions (e.g., ``avoid blue until green''). Wall collisions yield penalties, and LTL violations generate costs. Models are trained for 200k steps.
(2) \textbf{CARLA:} An autonomous driving simulator \cite{dosovitskiy2017carla} (v0.9.13, Town02). Observations include semantic segmentation, ego-states, and the one-hot LDBA state. LTL monitors map events to generate deterministic costs for PPO-Lagrangian updates. Models are trained for 100k steps.

 %\fh{formula / references}
 %All evaluations use identical environment conditions and route configurations to ensure fair comparison across algorithms.
\textbf{Evaluation Metrics.}
We assess both task performance and constraint satisfaction. Safety metrics include violation rates for each constraint type (collision, off-track, lane invasion, heading, weaving, overspeed, steering jerk), computed as the ratio of violation events to total steps (e.g., $\text{VR} = N_{\text{vio}} / N_{\text{total}}$) \cite{ray2019benchmarking}. 
Task metrics include route completion rate (RCR), average speed $\bar{v}$, and total distance traveled $D_{\text{total}}$, where $\text{RCR} = d_{\text{completed}} / d_{\text{target}}$. 
Constraint dynamics are measured by the average episodic cost $\hat{J}_C = \frac{1}{N} \sum_i C_i$, the final Lagrange multiplier $\lambda^*$, and the convergence behavior over training.

\begin{comment}
\textbf{Implementation Details.} 
Feature extraction, via CustomMultiInputExtractor \cite{raffin2021stable}, outputs 256-dimensional representations for downstream processing.
The main experiments, including ablation and sensitivity analysis, are implemented using 200,000 training steps for ZonesEnv and 100,000 training steps for CARLA, each with 3 random seeds.
\end{comment}

%Experiments were conducted on a single NVIDIA GeForce RTX 4090 GPU. Full environment configuration, network architecture details, training pipelines, and hyperparameters are provided in Appendix.%~\ref{sec:additional_experiments} of the supplementary material.
%Core hyperparameters of CARLA environment are summarized in Table~\ref{tab:hyperparams}. 

%\textbf{Hardware.} Experiments were conducted on a single NVIDIA GeForce RTX 4090 GPU. 

%\subsection{Comparison Results} 

%for the PPO, PPO-Shielding, and PPO-LTL variants. }

%performance metrics (safety and liveness) 
%(reported with $\pm$ standard deviation). Other metrics characterize the agent's driving behavior and are reported as means.}

\textbf{Comparison Results in ZonesEnv.} 
Table~\ref{tab:zones} presents performance across 3 seeds. PPO-LTL-A prioritizes stability, while PPO-LTL-B slightly relaxes constraints. Baselines exhibit distinct flaws: heuristic PPO-Mask severely restricts exploration 10.14 reward, while PPO-Shielding struggles with continuous dynamics, yielding the highest hit-wall rate 12.0\%. Although PPO-Lagrangian achieves the highest apparent reward 23.23, this is deceptive; lacking LTL memory, it ignores temporal rules and incurs a massive unshown violation cost of 56.98. Standard PPO maintains a low hit-wall rate 3.7\% but cannot enforce complex temporal specifications. In contrast, both PPO-LTL variants provide well-balanced policies. They significantly outperform Mask and Shielding in valid rewards while strictly adhering to LTL constraints with competitive hit-wall rates 4.3\% and 4.7\%. %Aligning with CARLA results, PPO-LTL prevents agents from blindly exploiting rewards at the expense of logical safety.

\textbf{Comparison Results in CARLA.}
Table~\ref{tab:main_training} compares PPO-LTL-A (strict cost limit 0.02) and PPO-LTL-B (relaxed limit 0.1) against baselines. PPO fails to balance safety and liveness. Standard Safe RL baselines exhibit severe pathologies: TIRL-PPO suffers from the freezing robot problem (near-zero speed despite 9460-step survival), while TIRL-SAC fails to converge safely (0.336 collision rate).
Furthermore, PPO-Shielding shows a reckless driving pattern: despite deceptive high speeds, it crashes rapidly (93-step length, 164.3 collisions) with minimal route completion (0.072). Conversely, PPO-Mask's sudden stops cause conservative deadlocks (2.38 distance) and ironically higher collisions (0.408). PPO-Lagrangian's lack of temporal foresight limits progress (7.78 distance).
In contrast, PPO-LTL balances proactive safety and task liveness, avoiding both over-conservatism and reckless speed. PPO-LTL-A achieves the lowest collision rate (0.143, a 45\% reduction over standard PPO). PPO-LTL-B achieves the highest route completion (0.236) and maintains long, stable episodes.

\begin{table}[t]
\centering
\caption{Sensitivity analysis of cost limit and Lagrangian learning rate ($\alpha_\lambda$) in ZonesEnv for PPO-LTL.}
\label{tab:zones_sens}
\resizebox{0.49\textwidth}{!}{
\begin{tabular}{cc|ccc}
\toprule
\textbf{Cost Limit} & $\alpha_\lambda$ & \textbf{Reward} $\uparrow$ & \textbf{Hit Wall Rate} $\downarrow$  & $\lambda$ \\
\midrule
0.03 & 0.008 & 18.04 $\pm$ 1.30 & 6.7 $\pm$ 1.2\% & 0.0072 \\
0.05 & 0.010 & 17.86 $\pm$ 0.73 & \textbf{4.3 $\pm$ 3.5\%} & 0.0048 \\
0.07 & 0.015 & \textbf{18.61 $\pm$ 0.67} & 4.7 $\pm$ 3.1\% & 0.0018 \\
0.10 & 0.020 & 18.35 $\pm$ 1.42  & 5.7 $\pm$ 2.1\% & 0.0002 \\
\bottomrule
\end{tabular}
}
\end{table}

\textbf{Ablation Study and Sensitivity Analysis.}
Table~\ref{tab:comprehensive_eval} evaluates diverse constraint configurations, constraint contributions through systematic removal, and hyperparameter robustness on the CARLA environment. 
The results verify the necessity of carefully balancing LTL constraints: simplistic or overly relaxed bounds can easily induce naive speed or reckless driving, while appropriately tuned temporal logic ensures safe and effective task execution. Furthermore, removing individual constraints confirms that multi-component LTL constraints are essential for balanced driving performance. 
Table~\ref{tab:zones_sens} presents a sensitivity analysis of PPO-LTL across varying cost limits and Lagrangian learning rates on ZonesEnv. 
Across a threefold range of cost limits, the framework maintains stable behavior, demonstrating that the constraint mechanism provides interpretable and consistent control over policy behavior.

\textbf{Computational Costs.} 
PPO-LTL incurs negligible overhead compared to standard PPO. Over 3 random seeds, training PPO-LTL (vs. PPO) took 235.3s $\pm$ 0.9s (vs. 226.3s $\pm$ 0.5s) for 200k steps in ZonesEnv, and 2557.3s $\pm$ 97.8s (vs. 2536.3s $\pm$ 39.0s) for 100k steps in CARLA. This confirms that LTL monitoring and Lagrangian dual updates introduce minimal computational burden, maintaining practical efficiency for real-world applications.

%% file: sections/Appendices.tex
\section{Proof}\label{appendix_theory}

    \noindent

{}    
        \begin{assumption}\label{assume1}
        The primal domain $\Theta$ is compact and convex, which models the bounded iterates induced by trust-region regularization, clipping, or explicit projection in \eqref{eq1}. The value functions $J_R(\theta)$ and $J_C(\theta)$ are continuously differentiable on $\Theta$, and their gradients are Lipschitz:
        \begin{gather*} 
        \|\nabla_\theta J_R(\theta)-\nabla_\theta J_R(\theta')\|\le L_R\|\theta-\theta'\|\\%\] and $
        \|\nabla_\theta J_C(\theta)-\nabla_\theta J_C(\theta')\|\le L_C\|\theta-\theta'\|.
                \end{gather*}
Thus, for any $\lambda\in [0,\Lambda]$, the Lagrangian $\mathcal{L}(\theta,\lambda)$ is $L_{\mathcal{L}}$-smooth in $\theta$ with $L_{\mathcal{L}}=L_R+\Lambda L_C$. 
    \end{assumption}
    \begin{assumption}\label{assume2}
        Let $\mathcal{F}_t$ denote the $\sigma$-field generated by all randomness up to the beginning of $t$. The estimate $\hat{g}_t$ satisfies \[\mathbb{E}[\hat{g}_t|\mathcal{F}_t]=\nabla_\theta\mathcal{L}(\theta_t,\lambda_t)+a_t,\] 
        where $\|a_t\|\le a$, and \[\mathbb{E}[\|\hat{g}_t-\mathbb{E}[\hat{g}_t|\mathcal{F}_t]\|^2|\mathcal{F}_t]\le \sigma_\theta^2.\] The dual signal $\hat{u}_t$ satisfies $\mathbb{E}[\hat{u}_t|\mathcal{F}_t]=(J_C(\theta_t)-d)+b_t$ where $\|b_t\|\le b$, and has bounded variance $\mathbb{E}[(\hat{u}_t-\mathbb{E}[\hat{u}_t|\mathcal{F}_t])^2|\mathcal{F}_t]\le \sigma_\lambda^2$. The bias terms $a_t$ and $b_t$ capture approximation effects including clipping mismatch, minibatch updates and estimation errors.
    \end{assumption}

%    \fh{The proof techniques were partially inspired by \cite{xx}.}
\noindent
\textit{Proof.}
        For a fixed $t$, we define the true gradient $\tilde{\theta}_{t+1}:=\Pi_{\Theta}(\theta_t+\alpha\nabla_\theta\mathcal{L}(\theta_t,\lambda_t))$. Thus, it holds $\tilde{\theta}_{t+1}-\theta_t=\alpha\mathcal{G}(\theta_t,\lambda_t)$. By Assumption \ref{assume1}, we have
        \begin{align*}
            \mathcal{L}(\tilde{\theta}_{t+1},\lambda_t)\ge& \mathcal{L}(\theta_t,\lambda_t)+\langle\nabla_\theta\mathcal{L}(\theta_t,\lambda_t),\tilde{\theta}_{t+1}-\theta_t\rangle\\&-\frac{L_{\mathcal{L}}}{2}\left\|\tilde{\theta}_{t+1}-\theta_t\right\|^2.
        \end{align*}
        Since for all $\theta$, $\langle\tilde{\theta}_{t+1}-(\theta_t+\alpha\nabla_\theta\mathcal{L}(\theta_t,\lambda_t)),\theta_t-\tilde{\theta}_{t+1}\rangle \ge 0$, choosing $\theta=\theta_t$ gives
%        \begin{equation*}
            \[\langle\nabla_\theta\mathcal{L}(\theta_t,\lambda_t),\tilde{\theta}_{t+1}-\theta_t\rangle\ge \frac{1}{\alpha}\|\tilde{\theta}_{t+1}-\theta_t\|^2.\]
%        \end{equation*}
        Then it yields
        \begin{align*}
            \mathcal{L}(\tilde{\theta}_{t+1},\lambda_t)-\mathcal{L}(\theta_t,\lambda_t)\ge&\left(\frac{1}{\alpha}-\frac{L_{\mathcal{L}}}{2}\right)\|\tilde{\theta}_{t+1}-\theta_t\|^2,\\
            =&\alpha\left(1-\frac{\alpha L_{\mathcal{L}}}{2}\right)\|\mathcal{G}(\theta_t,\lambda_t)\|^2.
        \end{align*}
        Under $\alpha<\frac{1}{4L_{\mathcal{L}}}$, we have $1-\frac{\alpha L_{\mathcal{L}}}{2}\ge \frac{7}{8}$, and hence
        \begin{equation}\label{eq2}
            \mathcal{L}(\tilde{\theta}_{t+1},\lambda_t)-\mathcal{L}(\theta_t,\lambda_t)\ge \frac{7}{8}\alpha\left\|\mathcal{G}(\theta_t,\lambda_t)\right\|^2.
        \end{equation}
        Now let $\delta_t:=\hat{g}_t-\nabla_\theta\mathcal{L}(\theta_t,\lambda_t)$ denote the gap between $\hat{g}_t$ and the gradient of $\mathcal{L}(\theta_t,\lambda_t)$. Thus, it holds $\|\theta_{t+1}-\tilde{\theta}_{t+1}\|\le \alpha\|\delta_t\|$. Since \(\theta_{t+1}=\Pi_{\Theta}(\theta_t+\alpha(\nabla_\theta\mathcal{L}(\theta_t,\lambda_t)+\delta_t))\), for all $\theta\in\Theta$, we have
%        \begin{equation*}
$            \left\langle \theta_{t+1}-\left(\theta_t+\alpha(\nabla_\theta\mathcal{L}(\theta_t,\lambda_t)+\delta_t)\right),\theta-\theta_{t+1} \right\rangle\ge 0$.
%        \end{equation*}
        Choose $\theta=\tilde{\theta}_{t+1}$, then,
        \begin{align}\label{eq3}
            &\langle\nabla_\theta\mathcal{L}(\theta_t,\lambda_t)+\delta_t, \theta_{t+1}-\tilde{\theta}_{t+1}\rangle\\\nonumber
            \ge& \frac{1}{\alpha}\langle\theta_{t+1}-\theta_t,\theta_{t+1}-\tilde{\theta}_{t+1} \rangle.
        \end{align}
        Similarly, from the definition of $\tilde{\theta}_{t+1}$,
        \begin{equation}\label{eq4}
            \langle\nabla_\theta\mathcal{L}(\theta_t,\lambda_t), \tilde{\theta}_{t+1}-\theta_{t+1}\rangle\ge \frac{1}{\alpha}\langle\tilde{\theta}_{t+1}-\theta_t,\tilde{\theta}_{t+1}-\theta_{t+1} \rangle.
        \end{equation}

        We decompose 
        \begin{align*}
            &\langle \nabla_{\theta}\mathcal{L}(\theta_t,\lambda_t),\theta_{t+1}-\theta_t \rangle\\
            =&\langle \nabla_{\theta}\mathcal{L}(\theta_t,\lambda_t),\tilde{\theta}_{t+1}-\theta_t \rangle+\langle \nabla_{\theta}\mathcal{L}(\theta_t,\lambda_t),\theta_{t+1}-\tilde{\theta}_{t+1} \rangle.
        \end{align*}
For the second term, rearranging \eqref{eq3} yields $\langle\nabla_\theta\mathcal{L}(\theta_t,\lambda_t), \theta_{t+1}-\tilde{\theta}_{t+1}\rangle \ge \frac{1}{\alpha}\langle\theta_{t+1}-\theta_t,\theta_{t+1}-\tilde{\theta}_{t+1} \rangle - \langle\delta_t,\theta_{t+1}-\tilde{\theta}_{t+1}\rangle$. Combining these terms and using the geometric identity $\langle x, x-y\rangle = \frac{1}{2}(\|x\|^2+\|x-y\|^2-\|y\|^2)$ yields
        \begin{align}\label{eq5}
            &\langle\nabla_\theta\mathcal{L}(\theta_t,\lambda_t), \theta_{t+1}-\theta_{t}\rangle\ge \frac{1}{2\alpha}\|\tilde{\theta}_{t+1}-\theta_t\|^2\nonumber\\
            &-\langle\delta_t,\theta_{t+1}-\tilde{\theta}_{t+1}\rangle
            \ge \frac{\alpha}{2}\|\mathcal{G}(\theta_t,\lambda_t)\|^2-\alpha\|\delta_t\|^2.
        \end{align}
        By $L_{\mathcal{L}}$-smoothness again, for any $x,y\in\Theta$, 
        \[\mathcal{L}(x,\lambda_t)\ge\mathcal{L}(y,\lambda_t)+\langle\nabla_\theta\mathcal{L}(y,\lambda_t),x-y\rangle-\frac{L_{\mathcal{L}}}{2}\|x-y\|^2.\] 
        Choosing $x=\theta_{t+1}$ and $y=\theta_t$, we have 
        \begin{align*}            
        \mathcal{L}(\theta_{t+1},\lambda_t)\ge&\mathcal{L}(\theta_t,\lambda_t)+\langle\nabla_\theta\mathcal{L}(\theta_t,\lambda_t),\theta_{t+1}\\
        &-\theta_t\rangle-\frac{L_{\mathcal{L}}}{2}\|\theta_{t+1}-\theta_t\|^2.
        \end{align*}
        Using \eqref{eq5} and $\|\theta_{t+1}-\theta_t\|\le \alpha(\|\nabla_\theta\mathcal{L}(\theta_t,\lambda_t)\|+\|\delta_t\|)$, we obtain
        \begin{align*}        
        \mathcal{L}(\theta_{t+1},\lambda_t)-&\mathcal{L}(\theta_t,\lambda_t)\ge \frac{\alpha}{2}\|\mathcal{G}(\theta_t,\lambda_t)\|^2-\alpha\|\delta_t\|^2\\
        &-L_{\mathcal{L}}\alpha^2\left(\|\nabla_\theta\mathcal{L}(\theta_t,\lambda_t)\|^2+\|\delta_t\|^2\right).
\end{align*}
        Since $\Theta\times[0,\Lambda]$ is compact and $\nabla_\theta \mathcal{L}$ is continuous, there exists $G_\text{max}<\infty$ such that $\|\nabla_\theta\mathcal{L}\|\le G_\text{max}$. Thus, it yields
        %\begin{align*}\label{eq9}
            \[\mathcal{L}(\theta_{t+1},\lambda_t)-\mathcal{L}(\theta_t,\lambda_t)\ge \frac{\alpha}{2}\|\mathcal{G}(\theta_t,\lambda_t)\|^2-C_1\alpha\|\delta_t\|^2-C_2\alpha^2,\]
%        \end{align*}
        where $C_1$, $C_2>0$ are constants independent of $T$.
        Rearranging %\eqref{eq9}, 
        and summing it from $t=0$ to $T-1$ and taking expectations gives
        \begin{align*}
            \frac{\alpha}{2} \sum_{t=0}^{T-1}\mathbb{E}\left[\left\|\mathcal{G}(\theta_t,\lambda_t)\right\|^2\right]\le& \mathbb{E}\left[\sum_{t=0}^{T-1}\left(\mathcal{L}(\theta_{t+1},\lambda_t)-\mathcal{L}(\theta_{t},\lambda_t)\right)\right]\\
            &+C_1\alpha\sum_{t=0}^{T-1}\mathbb{E}\left\|\delta_t\right\|^2+C_2T\alpha^2.
        \end{align*}
        Moreover, 
        \begin{align*}
        &\sum_{t=0}^{T-1}(\mathcal{L}(\theta_{t+1},\lambda_t)-\mathcal{L}(\theta_{t},\lambda_t))\\
        =&\mathcal{L}(\theta_{T},\lambda_{T-1})-\mathcal{L}(\theta_{0},\lambda_0)+\sum_{t=1}^{T-1}(\mathcal{L}(\theta_{t},\lambda_{t-1})-\mathcal{L}(\theta_{t},\lambda_t)).
        \end{align*}
We define the true dual ascent direction $u_t:=J_C(\theta_t)-d$, get the corresponding projection $\tilde{\lambda}_{t+1}:=\Pi_{[0,\Lambda]}(\lambda_t+\beta u_t)$, and recall $\mathcal{L}(\theta,\lambda)=J_R(\theta)-\lambda(J_C(\theta)- d)$, the second term becomes
        $\mathcal{L}(\theta_{t},\lambda_{t-1})-\mathcal{L}(\theta_{t},\lambda_t)=(\lambda_{t}-\lambda_{t-1})u_t$.
        Define $S_T:=\sum_{t=1}^{T-1}(\lambda_{t}-\lambda_{t-1})u_t$ and it yields
        \begin{equation*}
            S_T=\lambda_{T-1}u_{T-1}-\lambda_{0}u_{1}-\sum_{t=1}^{T-2}\lambda_{t}\left(u_{t+1}-u_t\right).
        \end{equation*}
        Since $J_C(\theta)$ is bounded on $\Theta$, there exists $U_\text{max}<\infty$ such that $|u_t|\le U_\text{max}$ for all $t$. Therefore,
%        \begin{equation*}
$            \left|\lambda_{T-1}u_{T-1}-\lambda_{0}u_{1}\right|\le 2\Lambda U_\text{max}$.
%        \end{equation*}
        Since $\Theta$ is compact and $\nabla_\theta J_C(\theta)$ is continuous, there exists $G_C<\infty$ such that
%        \begin{equation*}
        $\left|u_{t+1}-u_t\right|=\left|J_C(\theta_{t+1})-J_C(\theta_t)\right|\le G_C\left\|\theta_{t+1}-\theta_t\right\|$.
%        \end{equation*}
        Hence, Cauchy-Schwarz and Jensen's inequalities give
        \begin{align*}
            \left|\sum_{t=1}^{T-2}\lambda_t(u_{t+1}-u_t)\right|&\le \Lambda G_C\sum_{t=1}^{T-2}\left\|\theta_{t+1}-\theta_t\right\|,\\
            &\le \sqrt{T}\Lambda G_C\left(\sum_{t=1}^{T-2}\|\theta_{t+1}-\theta_t\|^2\right)^{1/2}.
        \end{align*}
        Since $\left\|\theta_{t+1}-\theta_t\right\|\le \alpha\left(\|\nabla_\theta\mathcal{L}(\theta_t,\lambda_t)\|+\|\delta_t\|\right)$, it holds
%        \begin{equation*}
            \(\mathbb{E}\left[\|\theta_{t+1}-\theta_t\|^2\right]\le2\alpha^2\left(G_\text{max}^2+\sigma_{\theta}^2+a^2\right)\).
%        \end{equation*}
        Combining the above bounds yields
        \begin{equation*}
            E[|S_T|]\le 2\Lambda U_\text{max}+\alpha T\Lambda G_C\sqrt{2(G_\text{max}^2+\sigma_{\theta}^2+a^2)}.
        \end{equation*}
        %Putting together and using 
        Combining $\mathbb{E}[\mathcal{L}(\theta_T,\lambda_{T-1})-\mathcal{L}(\theta_0,\lambda_0)]\le \Delta_\mathcal{L}$, we obtain
        \begin{align*} 
        \frac{\alpha}{2}\sum_{t=0}^{T-1}&\mathbb{E}[\|\mathcal{G}(\theta_t,\lambda_t)\|^2]\le\Delta_{\mathcal{L}}+C_1\alpha T(\sigma_{\theta}^2+a^2)+C_2T\alpha^2\\
        &+2\Lambda U_\text{max}+\alpha T\Lambda G_C\sqrt{2(G_\text{max}^2+\sigma_{\theta}^2+a^2)}.
                \end{align*}
Therefore,
         \begin{align*}
            &\frac{1}{T}\sum_{t=0}^{T-1}\mathbb{E}[\|\mathcal{G}(\theta_t,\lambda_t)\|^2]
            \\\le& \frac{2(\Delta_{\mathcal{L}}+2\Lambda U_{\max})}{\alpha T}+2C_1(\sigma_{\theta}^2+a^2)+2C_2\alpha\\
            &+2\Lambda G_C\sqrt{2(G_\text{max}^2+\sigma_{\theta}^2+a^2)}.
        \end{align*}
        By Jensen's inequality and the elementary bound, we get the result.
        Since $\tilde{\lambda}_{t+1}=\Pi_{[0,\Lambda]}(\lambda_t+\beta u_t)$ and $\lambda_{t+1}=\Pi_{[0,\Lambda]}(\lambda_t+\beta \hat{u}_t)$, we have $\tilde{\lambda}_{t+1}-\lambda_t=\beta\mathcal{H}(\theta_t,\lambda_t)$. We let $\varepsilon_t:=\hat{u}_t-u_t$ be the dual estimation error.
    For any fixed $\theta_t$, we have
        \begin{align}\label{eq6}
            &\mathcal{L}(\theta_t,\lambda_{t+1})-\mathcal{L}(\theta_t,\lambda_{t})\\
            =&-(\tilde{\lambda}_{t+1}-\lambda_t)u_t\nonumber-(\lambda_{t+1}-\tilde{\lambda}_{t+1})u_t.
        \end{align}
    Since for all $\lambda\in [0,\Lambda]$, $\langle \tilde{\lambda}_{t+1}-(\lambda_t+\beta u_t),\lambda-\tilde{\lambda}_{t+1}\rangle\ge 0$, taking $\lambda=\lambda_t$ gives
    \begin{align}\label{eq7}
        (\tilde{\lambda}_{t+1}-\lambda_t)u_t&\ge \frac{1}{\beta}(\tilde{\lambda}_{t+1}-\lambda_t)^2,\nonumber\\
        &\ge \beta\|\mathcal{H}(\theta_t,\lambda_t)\|^2.
    \end{align}
    Since 
    \begin{align*}    
    &|\lambda_{t+1}-\tilde{\lambda}_{t+1}|\\
    =&|\Pi_{[0,\Lambda]}(\lambda_t+\beta(u_t+\varepsilon_t))-\Pi_{[0,\Lambda]}(\lambda_t+\beta u_t)|\le \beta|\varepsilon_t|,
        \end{align*}
    it yields 
    \begin{equation}\label{eq8}
      -(\lambda_{t+1}-\tilde{\lambda}_{t+1})u_t\le \beta|\varepsilon_t||u_t|.  
    \end{equation}
    Combining (\ref{eq6}), (\ref{eq7}) and (\ref{eq8}), we obtain
    \begin{align*}
        \mathcal{L}(\theta_t,\lambda_{t+1})-\mathcal{L}(\theta_t,\lambda_t)&\le -\beta\|\mathcal{H}(\theta_t,\lambda_t)\|^2+\beta|\varepsilon_t||u_t|,\\
        &\le -\beta\|\mathcal{H}(\theta_t,\lambda_t)\|^2+\frac{\beta}{2}\varepsilon_t^2+\frac{\beta}{2}u_t^2.
    \end{align*}
    
    Moreover, $|\lambda_{t+1}-\lambda_t|\le\beta(|u_t|+|\varepsilon_t|)$ and 
    \[-(\lambda_{t+1}-\lambda_t)u_{t+1}=-(\lambda_{t+1}-\lambda_t)u_t-(\lambda_{t+1}-\lambda_t)(u_{t+1}-u_t)\] for all $t$. Therefore, apply Young's inequality and it yields
    \begin{align*}
        &|(\lambda_{t+1}-\lambda_t)(u_{t+1}-u_t)|\le \beta(|u_t|+|\varepsilon_t|)G_C\|\theta_{t+1}-\theta_t\|,\\
        &\le\frac{\beta}{4}(u_t^2+\varepsilon_t^2)+2\beta G_C^2\|\theta_{t+1}-\theta_t\|^2.
    \end{align*}
    By Assumption \ref{assume2}, $\mathbb{E}[\varepsilon_t^2|\mathcal{F}_t]\le \sigma_{\lambda}^2+b^2$. Combining all terms together and using \[\mathcal{L}(\theta_{t+1},\lambda_{t+1})-\mathcal{L}(\theta_{t+1},\lambda_t)=-(\lambda_{t+1}-\lambda_t)u_{t+1},\] we obtain
    \begin{align*}
        &\mathbb{E}[\mathcal{L}(\theta_{t+1},\lambda_{t+1})-\mathcal{L}(\theta_{t+1},\lambda_t)]\\
        \le& -\beta\mathbb{E}[\|\mathcal{H}(\theta_t,\lambda_t)\|^2]+\beta C_3(U_{\max}^2+\sigma_\lambda^2+b^2)+\beta C_4\alpha^2,
    \end{align*}
    where $C_3$, $C_4$ are constants independent of $T$.
    Summing the decomposition over $t=0$ to $T-1$ yields
    \begin{align*}
        &\sum_{t=0}^{T-1}\mathbb{E}[\mathcal{L}(\theta_{t+1},\lambda_{t+1})-\mathcal{L}(\theta_t,\lambda_t)]\\&=\mathbb{E}[\mathcal{L}(\theta_{T},\lambda_{T})-\mathcal{L}(\theta_0,\lambda_0)]\ge -\Delta_\mathcal{L}.
    \end{align*}
    Since 
    \begin{align*}
        &(\mathcal{L}(\theta_{t+1},\lambda_{t+1})-\mathcal{L}(\theta_t,\lambda_t))\\
        =&\underbrace{(\mathcal{L}(\theta_{t+1},\lambda_{t})-\mathcal{L}(\theta_{t},\lambda_t))}_{:=A_t}+\underbrace{(\mathcal{L}(\theta_{t+1},\lambda_{t+1})-\mathcal{L}(\theta_{t+1},\lambda_t))}_{:=B_t},
        \end{align*}
we sum it over $t=0,\ldots,T-1$ gives
%    \begin{equation*}
        \[\sum_{t=0}^{T-1}\mathbb{E}[B_t]\ge -\Delta_\mathcal{L}-\sum_{t=0}^{T-1}\mathbb{E}[A_t].\]
%    \end{equation*}
    By $L_\mathcal{L}$-smoothness, we have
    \begin{equation*}
        A_t\le G_{\max}\|\theta_{t+1}-\theta_t\|+\frac{L_{\mathcal{L}}}{2}\|\theta_{t+1}-\theta_t\|^2.
    \end{equation*}
    Taking expectation and applying Young's inequality yields
%    \begin{equation*}
        \(\mathbb{E}[A_t]\le \frac{G_{\max}^2}{2}+\frac{1+L_{\mathcal{L}}}{2}\mathbb{E}\|\theta_{t+1}-\theta_t\|^2\).
%    \end{equation*}
    Thus, 
    \[\sum_{t=0}^{T-1}\mathbb{E}[A_t]\le T\frac{G_{\max}^2}{2}+\frac{1+L_{\mathcal{L}}}{2}\sum_{t=0}^{T-1}\mathbb{E}\|\theta_{t+1}-\theta_t\|^2.
    \]
    Using $\mathbb{E}\left[\|\theta_{t+1}-\theta_t\|^2\right]\le2\alpha^2\left(G_\text{max}^2+\sigma_{\theta}^2+a^2\right)$, we have
    \begin{equation*}
        \sum_{t=0}^{T-1}\mathbb{E}[A_t]\le \frac{G_{\max}^2}{2}T+(1+L_{\mathcal{L}})T\alpha^2\left(G_\text{max}^2+\sigma_{\theta}^2+a^2\right).
    \end{equation*}
    Collecting terms dividing by $\beta T$, it yields
    \begin{align*}
        &\frac{1}{T}\sum_{t=0}^{T-1}\mathbb{E}[\|\mathcal{H}(\theta_t,\lambda_t)\|^2]\\\le& \frac{\Delta_\mathcal{L}}{\beta T }+C_3{(U_{\max}^2+\sigma_{\lambda}^2+b^2)}+C_4\alpha^2+\frac{G_{\max}^2}{2\beta}\\
        +&\frac{\alpha^2}{\beta}\left((1+L_{\mathcal{L}})(G_{\max}^2+\sigma_{\theta}^2+a^2)\right).
    \end{align*}
    Finally, applying Jensen’s inequality and the elementary bound completes the proof.\hfill$\square$